\documentclass[11pt]{article}
\usepackage{amsfonts,amsmath,amsthm, mathtools, amssymb, mathabx}
\usepackage{fullpage}
\usepackage{scalerel}
\usepackage[utf8]{inputenc}
\usepackage{color}
\usepackage{enumitem}
\usepackage[style=alphabetic]{biblatex}
\usepackage{hyperref}
\usepackage{authblk}

\DeclareMathOperator*{\argmin}{arg\,min}
\DeclareMathOperator*{\argmax}{arg\,max}

\DeclareMathOperator*{\kl}{KL}

\def\bbE{\mathbb{E}}

\def\bbR{\mathbb{R}}

\def\calF{\mathcal{F}}
\def\calR{\mathcal{R}}

\def\calP{\mathcal{P}}
\def\calQ{\mathcal{Q}}
\def\calH{\lambda}

\def\pdata{P_{\scaleto{\textnormal{data}}{4pt}}}
\def\hatpdata{\widehat{P}_{\scaleto{\textnormal{data}}{4pt}}}

\newcommand{\bydef}{\stackrel{\bigtriangleup}{=}}
\newcommand*\diff{\mathop{}\!d}
\DeclarePairedDelimiter\abs{\lvert}{\rvert}
\DeclarePairedDelimiter\norm{\lVert}{\rVert}

\newtheorem{theorem}{Theorem}
\newtheorem{definition}[theorem]{Definition}
\newtheorem{fact}[theorem]{Fact}
\newtheorem{lemma}[theorem]{Lemma}
\newtheorem{corollary}[theorem]{Corollary}
\newtheorem{remark}[theorem]{Remark}

\title{The Inductive Bias of Restricted $f$-GANs}

\author[1]{Shuang Liu\thanks{shuangliu@ucsd.edu}}
\author[2]{Kamalika Chaudhuri\thanks{kamalika@cs.ucsd.edu}}
\affil[1, 2]{University of California, San Diego}

\begin{document}
\maketitle

\begin{abstract}
Generative adversarial networks are a novel method for statistical inference that have achieved much empirical success; however, the factors contributing to this success remain ill-understood. In this work, we attempt to analyze generative adversarial learning -- that is, statistical inference as the result of a game between a generator and a discriminator -- with the view of understanding how it differs from classical statistical inference solutions such as maximum likelihood inference and the method of moments. 

Specifically, we provide a theoretical characterization of the distribution inferred by a simple form of generative adversarial learning called restricted $f$-GANs -- where the discriminator is a function in a given function class, the distribution induced by the generator is restricted to lie in a pre-specified distribution class and the objective is similar to a variational form of the $f$-divergence. A consequence of our result is that for linear KL-GANs -- that is, when the discriminator is a linear function over some feature space and $f$ corresponds to the KL-divergence -- the distribution induced by the optimal generator is neither the maximum likelihood nor the method of moments solution, but an interesting combination of both. 
\end{abstract}

\section{Introduction}


Generative adversarial networks (GANs) ~\cite{goodfellow2014generative} are a novel method for statistical inference that have received a great deal of recent attention. Given input samples from a data distribution, inference is carried out in the form of a two-player game between a generator and a discriminator, which are usually neural networks with pre-specified architectures. The generator attempts to generate samples that progressively mimic the input data; the discriminator attempts to accurately discriminate between the input and samples produced by the generator. The game continues until the discriminator fails to detect if an instance comes from the input or is produced by the generator, at which point the generator is said to have learned the data distribution. 
 
While generative adversarial networks have achieved much empirical success, the factors contributing to their success remain a mystery. For example, even if we ignore finite sample and optimization issues, it is still unknown what the GAN solution looks like, and what its relationship is to classical statistical solutions such as maximum likelihood and method of moments. Properties of the solution are partially understood when the generator is unrestricted~\cite{Nowozin2016f, goodfellow2014generative, liu2017approximation} and can produce samples from any distribution. In practice, we always have {\em{model mismatch}} -- the class of distributions that the generator produce samples from is restricted, and the input data distribution usually does not lie in this class. In this case, the relationship between the generator class, the discriminator class and the output distribution remains ill-understood. 

In this paper, we consider this problem in the context of restricted $f$-GANs -- which are $f$-GANs~\cite{Nowozin2016f} where the discriminator belongs to a class of functions $H$. We provide a theoretical characterization of the solutions provided in these cases under model mismatch. Our analysis relies on the Fenchel-Moreau theorem and Ky Fan's minimax theorem, with subroutines heavily inspired by \cite{rockafellar1968integrals, rockafellar2015measures,  rockafellar2018risk}.

An important consequence of our result can be seen when we specialize it to linear KL-GANs -- $f$-GANs whose objective function correspond to the variational form of the KL-divergence, and whose discriminator class is the set of all functions linear over a pre-specified feature set. In this case, we show that the distribution induced by the optimal generator is neither the maximum likelihood nor the method of moments solution, but an interesting combination of both.

\section{Preliminaries}

The basic problem of statistical inference is as follows. We are given samples from an unknown underlying distribution $\pdata$. Let $\hatpdata$ denote the empirical distribution of the input samples, our goal is to find a distribution $Q$ in a distribution class $\calQ$ to approximate $\pdata$. 

The problem is typically solved by using an objective function $D(\hatpdata, Q)$ that measures how well $Q$ fits the data, and then finding a $Q^*$ as follows: 


\begin{align}
    Q^* = \argmin_{Q\in\calQ}D\left(\hatpdata, Q\right).\label{eq: distancefunction}
\end{align}

Here, large $D$ means that $Q$ fits the data poorly, and different choices of $D$ lead to different inference solutions. 

\subsection{Background: Maximum Likelihood and Method of Moments}

Most classical statistical literature has looked at two major categories of inference methods -- maximum likelihood estimation and the method of moments.

\paragraph{Maximum Likelihood Estimation.} In {\em maximum likelihood estimation} (MLE), the goal is to select the distribution in $\calQ$ that maximizes the likelihood of generating the data  $\hatpdata$. For ease of discussion, let us assume that there is a base measure on the instance space, and $\hatpdata(x)$ and $Q(x)$ are  density functions of $\hatpdata$ and $Q$ respectively at $x$ with respect to this base measure. The goal of maximum likelihood estimation is to find:
\begin{align*}
    Q^* = \argmax_{Q\in\calQ}\prod_{x\in\hatpdata}Q(x).
\end{align*}
Since $\hatpdata$ is fixed, this is equivalent to finding the minimizer of:
\begin{align*}
    \frac{1}{\abs{\hatpdata}}\left(\log\prod_{x\in\hatpdata}\hatpdata(x) - \log\prod_{x\in\hatpdata}Q(x)\right) &=  \frac{1}{\abs{\hatpdata}}\sum_{x\in\hatpdata}\log\frac{\hatpdata(x)}{Q(x)} \\
    &= \kl(\hatpdata, Q).
\end{align*}
Thus the objective function $D$ in \eqref{eq: distancefunction} for MLE is the KL-divergence.

\paragraph{Method of Moments.} An alternative method for statistical inference, which dates back to Chebyshev, and has recently seen renewed interest, is the method of moments. In the {\em generalized method of moments} (GMM) \cite{hansen1982large}, in addition to the data $\hatpdata$ and the distribution class $\calQ$, we are given a set of relevant feature functions $\varphi = (\varphi_1, \varphi_2, \cdots, \varphi_n)$ over the instance space. The goal is to find the minimizer:
\begin{align*}
    Q^* = \argmin_{Q\in\calQ}\,\norm[\Big]{\bbE_{\hatpdata}[\varphi] - \bbE_{Q}[\varphi]}_2.
\end{align*}
Thus, for GMM, the objective function $D(\hatpdata, Q)$ in \eqref{eq: distancefunction} is $ \norm[\Big]{\bbE_{\hatpdata}[\varphi] - \bbE_{Q}[\varphi]}_2$.

Our goal is to understand how the solutions provided by GANs relate to these two standard ways of doing inference. 

\subsection{\texorpdfstring{$f$-Divergences and $f$-GANs}{f-Divergences and f-GANs}}

For the rest of the paper, we assume that we have an underlying probability space $(\Omega, \Sigma)$; all distributions we consider below are measures over this space. 

\begin{definition}[$f$-divergence,  \cite{ali1966general, csiszar1967information}]
Suppose $f: (-\infty, \infty)\to(-\infty, \infty]$ is a lower semi-continuous convex function such that $f(1) = 0$, $f$ is finite in some neighbourhood of $1$, and $f(x) = \infty$ for any $x < 0$. Let $P$ and $Q$ be probability measures over $(\Omega, \Sigma)$ where $P$ is absolutely continuous with respect to $Q$. Then, the {\em $f$-Divergence} of $P$ from $Q$ is defined as:
\begin{align}
    D_f(P||Q) \bydef \int_{\Omega} f\left(\frac{\diff P}{\diff Q}\right)\diff Q.\label{eq: f-div-definition-old}
\end{align}
\end{definition}

 Let $f^*: \bbR\to[-\infty, \infty]$ be the convex conjugate function of $f$, given by: $f^*(s) = \sup_{x \in \bbR} x^{\intercal} s - f(x)$; it is well-known that the $f$-divergences also have a variational formulation \cite{keziou2003dual,nguyen2010estimating} under certain conditions:
\begin{align}
    D_f(P||Q) &= \sup_{h} \bbE_{x\sim P}[h(x)] - \bbE_{x\sim Q}[f^*(h(x))],\label{eq: informal}
\end{align}
where the supremum is taken over, informally speaking, all possible functions. More details will be discussed in Section~\ref{sec: dual-f}.

Inspired by this variational formulation, \cite{Nowozin2016f} introduces a family of GANs, called {\em f-GAN}s, that use an $f$-divergnece $D_f$ as the objective function $D$ in \eqref{eq: distancefunction}. Inference is then formulated as solving the following minimax problem:
\begin{align}
    Q^* = \argmin_{Q\in\calQ}D_f\left(\hatpdata||Q\right) \approx \argmin_{Q\in\calQ}\sup_{h\in H} \bbE_{x\sim \hatpdata}[h(x)] - \bbE_{x\sim Q}[f^*(h(x))],\label{eq: f-gan-form}
\end{align}
where $H$ is a sufficiently large function class.

The standard GAN~\cite{goodfellow2014generative} is a special case of \eqref{eq: f-gan-form}, where $f(x) = x\log x - (x + 1)\log(x + 1)$, which corresponds to the Jensen-Shannon Divergence.

\subsection{\texorpdfstring{Restricted $f$-divergences and Restricted $f$-GANs}{Restricted f-divergences and Restricted f-GANs}}
To reduce the sample requirement \cite{arora2017generalization}, one might want to restrict the discriminator class $H$ in \eqref{eq: f-gan-form} to be a relatively small function class. To this end, we define the {\em restricted $f$-divergence}\footnote{Note that \cite{ruderman2012tighter} also uses the term ``restricted $f$-divergence", but for a very different purpose.}
\begin{align}
    D_{f, H} (P||Q) \bydef \sup_{h\in H} \bbE_{x\sim P}[h(x)] - \bbE_{x\sim Q}[f^*(h(x))].\label{eq: hard}
\end{align}

In practice, the discriminator class is often implemented by a neural network~\cite{Nowozin2016f}, therefore f-GANs are in fact {\em restricted f-GANs} that solve the following minimax problem\begin{align}
    Q^* = \argmin_{Q\in\calQ}D_{f, H}\left(\hatpdata||Q\right) = \argmin_{Q\in\calQ}\sup_{h\in H} \bbE_{x\sim \hatpdata}[h(x)] - \bbE_{x\sim Q}[f^*(h(x))].
\end{align}

A special case of \eqref{eq: hard} is {\em linear f-divergence}, introduced in~\cite{liu2017approximation}. Specifically, given a vector of feature functions $\varphi = (\varphi_1, \varphi_2, \cdots, \varphi_n)$ over the data domain, let $A$ be a convex set of $\bbR^n$, define 
\begin{align*}
    D_{f, \varphi, A}(P||Q)\bydef \sup_{a\in A, b\in\bbR} \bbE_{x\sim P}\left[a^{\intercal}\varphi(x) + b\right] - \bbE_{x\sim Q}\left[f^*\left(a^{\intercal}\varphi(x) + b\right)\right].
\end{align*}
$f$-GANs that solve $\argmin_{Q\in\calQ} D_{f, \varphi, A}(\hatpdata||Q)$ are called {\em linear f-GANs}.

\section{Main Result}

We begin with stating our main result in its most general form. 

\subsection{Additional Notations}
 We start out by introducing some notation. 
Recall that we have an underlying probability space $(\Omega, \Sigma)$. Let $B(\Omega, \Sigma)$ be the set of all real-valued bounded and measurable functions on $(\Omega, \Sigma)$ equipped with the topology induced by the uniform norm. 

 We use $ba(\Omega, \Sigma)$ to denote the set of all bounded and finitely additive signed measures over $(\Omega, \Sigma)$, $\calP(\Omega, \Sigma)$ to denote the set of all finitely additive probability measures over $(\Omega, \Sigma)$, and $\calP_c(\Omega, \Sigma)$ to denote the set of all countably additive probability measures over $(\Omega, \Sigma)$. Note that $\calP_c(\Omega, \Sigma)\subseteq \calP(\Omega, \Sigma)\subseteq ba(\Omega, \Sigma)$.

For any $\mu$ and $\nu\in ba(\Omega, \Sigma)$, we write $\mu\ll \nu$ to denote that $\mu$ is absolutely continuous w.r.t. $\nu$; that is, for any $E\in\Sigma$, $\nu(E) = 0\implies \mu(E) = 0$. Furthermore, if both $\mu$ and $\nu$ are countably additive, we use $\frac{\diff \mu}{\diff \nu}$ to denote the Radon-Nikodym derivative.

We extend definition \eqref{eq: f-div-definition-old} such that now $P$ can be a {\em{finitely additive}} probability measure that is {\em not necessarily} absolutely continuous w.r.t. $Q$.  Formally, for any $P\in \calP(\Omega, \Sigma)$ and $Q\in\calP_c(\Omega, \Sigma)$, define
\begin{align}
    \bar{D}_f(P||Q) &\bydef \begin{dcases}
    D_f(P||Q),\hspace{3em}\text{if $P,Q\in \calP_c(\Omega, \Sigma)$ and $P\ll Q$,}\\
    \sup_{h\in B(\Omega, \Sigma)} \bbE_{x\sim P}[h(x)] - \bbE_{x\sim Q}[f^*(h(x))],  \hspace{2.3
em} \text{otherwise}
    \end{dcases}\label{eq: f-div-definition-new}\\
    & \stackrel{(a)}{=} \sup_{h\in B(\Omega, \Sigma)} \bbE_{x\sim P}[h(x)] - \bbE_{x\sim Q}[f^*(h(x))],\label{eq: to-be-justified}
\end{align} 
where the equality (a) is justified by Theorem~\ref{thm: extension} in Section~\ref{sec: dual-f}, which is a rigorous version of \eqref{eq: informal}.

\subsection{General Result}
We begin with a slight generalization of the definition of restricted $f$-divergences $D_{f, H}$ in \eqref{eq: hard}. Let the functional $\calH: B(\Omega, \Sigma)\to[-\infty, \infty]$ be a regularizer, we can define
\begin{align}
    D_{f, \calH} (P||Q) \bydef \sup_{h\in B(\Omega, \Sigma)} \bbE_{x\sim P}[h(x)] - \bbE_{x\sim Q}[f^*(h(x))] - \calH(h).\label{eq: soft}
\end{align}
To see why \eqref{eq: soft} is a more general definition than \eqref{eq: hard}, let $H\subseteq B(\Omega, \Sigma)$ and define $\calH_H$ to be
\begin{align*}
    \calH_H(h) \bydef \begin{dcases}
                   0, &\text{if $h\in H$,}\\
                   \infty, &\text{otherwise,}
               \end{dcases}
\end{align*}
then we have $D_{f, \calH_H} = D_{f, H}$.

An important property of the functional $\calH$ is {\em shift-invariance}.
\begin{definition}[shift invariant]
$\calH$ is said to be shift invariant if for any $h\in B(\Omega, \Sigma)$ and $b\in\bbR$, $\calH (h) = \calH (h + b)$.
\end{definition}

We are also interested in the convex conjugate of $\calH$, denoted by $\calH^*$. According to  Theorem~\ref{thm: dual} in the appendix,  the functional $\calH^*$, although defined on $B(\Omega, \Sigma)^*$ by definition, can be equivalently defined on $ba(\Omega, \Sigma)$ such that
\begin{align*}
    \calH^*(\mu) = \sup_{h\in B(\Omega, \Sigma)} \int_{\Omega} h \diff \mu - \calH(h). 
\end{align*}

We are now ready for our main result.

\begin{theorem}\label{thm: main}
If $\calH$ is convex and shift invariant, $P\in \calP(\Omega, \Sigma)$, and $Q\in\calP_c(\Omega, \Sigma)$, then
\begin{align*}
    D_{f, \calH}(P||Q) &= \inf_{P'\in \calP(\Omega, \Sigma)} \, \calH^*(P - P') + \bar{D}_f(P'||Q)\\
    &= \inf_{\substack{P'\in \calP(\Omega, \Sigma)\\P'\ll Q}} \, \calH^*(P - P') + \bar{D}_f(P'||Q)
\end{align*}
\end{theorem}
We remark here that when $Q$ has finite support and $\lambda$ takes value $\infty$ outside a RKHS space, Theorem~\ref{thm: main} basically reduces to Theorem 2 in \cite{ruderman2012tighter}; and in this special case the proof can be greatly simplified.

Returning to the special case of $D_{f, H}$, recall that in this case
\begin{align*}
    \calH(h) = \begin{dcases}
                   0, &\text{if $h\in H$,}\\
                   \infty, &\text{otherwise,}
               \end{dcases}
\end{align*}
and note that 
\begin{align*}
    \lambda^*(P - P') &= \sup_{h\in B(\Omega, \Sigma)}\bbE_{P - P'}[h] - \lambda(h)\\
    & = \sup_{h\in H}\bbE_{P}[h] - \bbE_{P'}[h],
\end{align*}
we have the following corollary.

\begin{corollary}\label{cor: hard}
If $H$ is a convex subset of $B(\Omega, \Sigma)$ and for any $h\in H$, $b\in\bbR$, we have $h + b\in H$,  then for any $P\in \calP(\Omega, \Sigma)$ and $Q\in\calP_c(\Omega, \Sigma)$, 
\begin{align*}
    D_{f, H}(P||Q) &= \inf_{P'\in \calP(\Omega, \Sigma)} \, \bar{D}_f(P'||Q) + \sup_{h\in H}\bbE_{P}[h] - \bbE_{P'}[h]\\
    &= \inf_{\substack{P'\in \calP(\Omega, \Sigma)\\P'\ll Q}} \, \bar{D}_f(P'||Q) + \sup_{h\in H}\bbE_{P}[h] - \bbE_{P'}[h] 
\end{align*}
\end{corollary}
We would like to point out that if we take $H$ to be $B(\Omega, \Sigma)$ in Corollary~\ref{cor: hard}, then
\begin{align*}
    \sup_{h\in H}\bbE_{P}[h] - \bbE_{P'}[h] = \begin{dcases}
    0, &\textnormal{if $P' = P$,}\\\infty, &\textnormal{otherwise,}
    \end{dcases}
\end{align*}
and we recover Theorem~\ref{thm: extension}.

\subsection{\texorpdfstring{Implication for Linear $f$-GANs}{Implication for Linear f-GANs}}

Finally, because of its importance, it is worth emphasizing the special case of linear $f$-GANs. Recall that linear f-GANs minimize the objective $\min_{Q\in\calQ} D_{f, \varphi, A}(\hatpdata||Q)$ where $\varphi = (\varphi_1, \varphi_2, \cdots, \varphi_n)$ and $A$ is a convex subset of $\bbR^n$. In this case, take $H$ in Corollary~\ref{cor: hard} to be $\left\{a^{\intercal}\varphi: a\in A\right\}$, then
\begin{align}
    \sup_{h\in H}\bbE_{P}[h] - \bbE_{P'}[h] = \sup_{a\in A}a^{\intercal}\left(\bbE_{P}[\varphi] - \bbE_{P'}[\varphi]\right).\label{eq: general-a}
\end{align}
In particular, if $A = \left\{x\in\bbR^n: \norm{x}_2\leq R\right\}$, then in \eqref{eq: general-a} using Cauchy-Schwarz inequality we have
\begin{align*}
    \sup_{a\in A}a^{\intercal}\left(\bbE_{P}[\varphi] - \bbE_{P'}[\varphi]\right) = R\cdot \norm{\bbE_{P}[\varphi] - \bbE_{P'}[\varphi]}_2.
\end{align*}
\begin{remark}
We would like to point out that more generally, for any $p, q\in [1, \infty]$ such that $1/p + 1/q = 1$, if $A = \left\{x\in\bbR^n: \norm{x}_p\leq R\right\}$, using H{\"o}lder's inequality, we have 
\begin{align*}
    \sup_{a\in A}a^{\intercal}\left(\bbE_{P}[\varphi] - \bbE_{P'}[\varphi]\right) = R\cdot \norm{\bbE_{P}[\varphi] - \bbE_{P'}[\varphi]}_q.
\end{align*}But to keep the discussion concise, we state the results with $p = 2$.
\end{remark}

As a consequence, we have the following corollary.

\begin{corollary}
If $A = \left\{x\in\bbR^n: \norm{x}_2\leq R\right\}$ where $R$ is a positive real number, then for any $P\in \calP(\Omega, \Sigma)$ and $Q\in\calP_c(\Omega, \Sigma)$, 
\begin{align*}
    D_{f, \varphi, A}(P||Q) &= \inf_{P'\in \calP(\Omega, \Sigma)} \, R\cdot \norm{\bbE_{P}[\varphi] - \bbE_{P'}[\varphi]}_2 + \bar{D}_f(P'||Q)\\
    &= \inf_{\substack{P'\in \calP(\Omega, \Sigma)\\P'\ll Q}} \, R\cdot \norm{\bbE_{P}[\varphi] - \bbE_{P'}[\varphi]}_2 + \bar{D}_f(P'||Q).
\end{align*}
\end{corollary}

Observe that when $f(x) =  x\ln(x)$ and $P'\ll Q$, $\bar{D}_f(P'||Q)$ is the extended  $\kl$-divergence, which we denote as $\widebar{\kl}(P'||Q)$. Specifically, 
\begin{align*}
    \widebar{\kl}(P'||Q)\bydef
    \begin{dcases}
    \kl(P'||Q),&\text{if $P\in \calP_c(\Omega, \Sigma)$},\\
     \sup_{h\in B(\Omega, \Sigma)} \bbE_{x\sim P}[h(x)] - \bbE_{x\sim Q}\left[\exp(h(x) - 1)\right],&\text{otherwise.}
    \end{dcases}
\end{align*}

Therefore, we have that when $f(x) =  x\ln(x)$ and $A = \left\{x\in\bbR^n: \norm{x}_2\leq R\right\}$,
\begin{align*}
    D_{f, \varphi, A}(P||Q) = \inf_{\substack{P'\in \calP(\Omega, \Sigma)\\P'\ll Q}} \, R\cdot \norm{\bbE_{P}[\varphi] - \bbE_{P'}[\varphi]}_2 + \widebar{\kl}(P'||Q).
\end{align*}

Contrasting with maximum likelihood and method of moments estimators, the linear KL-GANs are an interesting combination of both when there is model mismatch.  Table~\ref{table: cmp} provides a summary of the differences. 

\begin{table}
  \caption{Linear KL-GAN combines MLE and GMM}
  \label{table: cmp}
  \vspace{1em}
  \centering
 \begin{tabular}{c | c | c | c }

   & MLE & GMM & Linear KL-GAN \\ \hline
  \rule{0pt}{3.6ex}
  $D(\hatpdata, Q)$ & $\kl(\hatpdata, Q)$ & $\norm[\Big]{\bbE_{\hatpdata}[\varphi] - \bbE_{Q}[\varphi]}_2$ &   $\inf\limits_{\substack{P'\in \calP(\Omega, \Sigma)\\P'\ll Q}} R\cdot \norm[\Big]{\bbE_{\hatpdata}[\varphi] - \bbE_{P'}[\varphi]}_2 + \widebar{\kl}(P'||Q)$

\end{tabular}
\end{table}

It is also possible to consider the case where $R = \infty$, which means $A = \bbR^n$. In this case 
\begin{align*}
    \sup_{a\in A}a^{\intercal}\left(\bbE_{P}[\varphi] - \bbE_{P'}[\varphi]\right) = \begin{dcases}
    0, &\text{if $\bbE_{P'}[\varphi] = \bbE_{P}[\varphi]$},\\
    \infty, &\text{otherwise.}
    \end{dcases}
\end{align*}
This will result in the following corollary.
\begin{corollary}\label{cor: r-inf}
If $A = \bbR^n$, then for any $P\in \calP(\Omega, \Sigma)$ and $Q\in\calP_c(\Omega, \Sigma)$, 
\begin{align*}
    D_{f, \varphi, A}(P||Q) &= \inf_{\substack{P'\in \calP(\Omega, \Sigma)\\\bbE_{P'}[\varphi] = \bbE_{P}[\varphi]}} \,\bar{D}_f(P'||Q)\\
    &= \inf_{\substack{P'\in \calP(\Omega, \Sigma)\\\bbE_{P'}[\varphi] = \bbE_{P}[\varphi]\\P'\ll Q}} \,\bar{D}_f(P'||Q).
\end{align*}
\end{corollary}

\subsection{Variational Representation of f-divergences}\label{sec: dual-f}

In this section, we will explain why equality (a) in \eqref{eq: to-be-justified} holds. The following theorem, which is complementary to Theorem 2.1 in \cite{keziou2003dual} and Lemma 1 in \cite{nguyen2010estimating}\footnote{We would like to note two things here. First, the ``only if" part of Lemma 1 in \cite{nguyen2010estimating} is unproved, and does not hold, therefore our result does not contradict theirs. Second, while both \cite{keziou2003dual} and \cite{nguyen2010estimating} mention that the supremum can be attained at $\partial f\left(\frac{\diff P}{\diff Q}\right)$, this sub-differential may not exist (especially when $f$ can take value $\infty$), and even if is well-defined everywhere needed, it is possible that the sub-differential is not bounded, hence not in $B(\Omega, \Sigma)$; therefore, their results do not imply ours.}, gives a rigorous variational representation of the $f$-divergence.
 
\begin{theorem}\label{thm: extension}
For any probability measures $P$ and $Q$ over $(\Omega, \Sigma)$ such that $P$ is absolutely continuous with respect to $Q$,
\begin{align}
    D_f(P||Q) &= \sup_{h\in B(\Omega, \Sigma)} \bbE_{x\sim P}[h(x)] - \bbE_{x\sim Q}[f^*(h(x))].\label{eq: dual-f-div}
\end{align}
\end{theorem}

\section{Related Work}

As a novel method for statistical inference, generative adversarial networks~\cite{goodfellow2014generative} have sparked a great deal of follow-up work on both theoretical and empirical sides.

The work most relevant to us are~\cite{goodfellow2014generative, Nowozin2016f} and~\cite{liu2017approximation}. \cite{goodfellow2014generative} shows that when both generators and discriminators are unrestricted, the optimal GAN solution converges to the input data distribution. \cite{Nowozin2016f} introduces $f$-GANs -- given samples $\hatpdata$ from an unknown data distribution $\pdata$, the objective is to find a distribution $Q$ that minimizes $D_f\left(\hatpdata, Q\right)$, where $D_f$ is an $f$-divergence. They show that minimizing this objective is equivalent to a GAN where the discriminators are unrestricted, and the objective corresponds to the variational form of the relevant $f$-divergence. 

\cite{liu2017approximation} considers approximation properties of GANs when the discriminators are restricted, but the input distribution lies in the interior of the class of distributions that can be produced by the generators -- in short, there is no model mismatch. They show that in this case, the solution produced by linear $f$-GANs -- that is, $f$-GANs whose discriminators are linear over a pre-specified feature space $\phi$ -- have the property that: $\bbE_{x \sim \hatpdata}[\phi(x)] = \bbE_{x \sim Q}[\phi(x)]$. In other words, the optimal solution agrees with the generalized method of moments solution. Our work can be thought of as an extension of this work to the model mismatch case. \cite{nock2017f} provides an information-geometric characterization of $f$-GANs when the input and the generator belongs to a class of distributions called the deformed exponential family. 


On the theoretical side,~\cite{arora2017generalization, singh2018nonparametric, liang2017well, bai2018approximability, feizi2017understanding} consider finite sample issues in GANs under different objective functions in various parametric and non-parametric settings, and provide bounds on their sample requirement. \cite{biau2018some} provides asymptotic convergence bounds on GAN solutions when both generators and discriminators are unrestricted. \cite{bottou2018geometrical} provide an analysis of the geometry of different GAN objective functions, with a view towards explaining their relative performance.

Finally, there has also been much recent work on the theoretical analysis of the optimization challenges that arise in the inference process of GANs; some examples include \cite{heusel2017gans, nagarajan2017gradient, li2017towards, mescheder2017numerics, barnett2018convergence}.

\section{Conclusion}
In conclusion, we provide a theoretical characterization of the distribution induced by the optimal generator in generative adversarial learning. Unlike prior work~\cite{goodfellow2014generative,liu2017approximation}, our result applies when both the generator and the discriminator are restricted. When applied to linear $f$-GANs, our characterization shows that the optimal linear KL-GAN solution offers an interesting mix of maximum likelihood and the method of moments.

Our work assumes that a sufficient number of samples is always available and that the optimal solution is always attainable. We believe removing these assumptions is an important avenue for future work.

\paragraph{Acknowledgments.} We thank NSF under IIS 1617157 and ONR under N00014-16-1-261 for research support.

\printbibliography

\appendix
\section{Preliminaries for the Proofs}
For any $E\in\Sigma$, denote by $1_E: \Omega\to\bbR$ the indicator function that takes value $1$ over $E$ and $0$ everywhere else. We will sometimes use constants to represent constant functions. For any two real-valued functions $g$ and $g'$ defined over the same domain $D$, we write $g\leq g'$ if for any $x\in D$, $g(x) \leq g'(x)$. 
For any topological vector space $X$, we denote by $X^*$ the topological dual of $X$, which is the set of all continuous linear functions over $X$.

\begin{theorem}[dual of $B(\Omega, \Sigma)$~\cite{hildebrandt1934bounded}]\label{thm: dual} $B(\Omega, \Sigma)^*$ can be identified with $ba(\Omega, \Sigma)$ by defining for any $h\in B(\Omega, \Sigma)$ and any $\mu\in ba(\Omega, \Sigma)$
\begin{align*}
    \left<h, \mu\right> \bydef \int_{\Omega}h\diff\mu.
\end{align*}
\end{theorem}

\begin{definition}[general convex conjugacy \cite{rockafellar1968integrals}]\label{def: conjugacy}
Let $(E, E')$ be a pair of real vector spaces, $\left<x, x'\right>$ be a real bilinear function of $x\in E$ and $x'\in E'$, and $\calF: E\to (-\infty, \infty]$ be a proper convex function, then we can define on $E'$ the conjugate of $\calF$, denoted by $\calF^*$, as
\begin{align*}
    \calF^*(x^*) \bydef \sup_{x\in E} \left<x, x'\right> - \calF(x),
\end{align*}
and define on $X$ the conjugate of $\calF^*$, denoted by $\calF^{**}$, as 
\begin{align*}
    \calF^{**}(x) \bydef \sup_{x'\in E'} \left<x, x'\right> - \calF^*(x);
\end{align*}
if only $\calF$ is specified, then it is assumed that $E$ is the domain of $\calF$ and $E'$ is $E^*$, and the bilinear function is given by $\left<x, x^*\right>\bydef x^*(x)$ for $x\in E$ and $x^*\in E^*$.
\end{definition}

\begin{theorem}[Fenchel-Moreau,  \cite{zalinescu2002convex} Theorem 2.3.3]\label{thm: fm}
If $E$ is a Hausdorff locally convex space, and $\calF$ is a proper lower semi-continuous convex function on $E$, then $\calF = \calF^{**}$. 
\end{theorem}

\begin{fact}\label{fact: r-hausdorff}
$\bbR$ with the usual topology is a Hausdorff locally convex space.
\end{fact}
\begin{proof}
The usual topology on $\bbR$ can be induced by the usual norm on $\bbR$ and a normed space is a Hausdorff locally convex space.
\end{proof}
\begin{fact}\label{fact: b-hausdorff}
$B(\Omega, \Sigma)$ is a Hausdorff locally convex space.
\end{fact}
\begin{proof}
The topology on $B(\Omega, \Sigma)$ is induced from the uniform norm and a normed space is a Hausdorff locally convex space.
\end{proof}
\begin{fact}\label{fact: f}
$f$ is a proper lower semi-continuous convex function.
\end{fact}
\begin{proof}
Recall that by assumption $f$ is a lower semi-continuous convex function. To see $f$ is also proper, note that by assumption $f(1) = 0$, therefore $f$ is a lower semi-continuous convex function that takes finite value at some point, hence also a proper function.
\end{proof}
\begin{fact}\label{fact: fstar}
$f^*$ is a proper lower semi-continuous convex function.
\end{fact}
\begin{proof}
Note that $\bbR^* = \bbR$, and the weak* topology on $\bbR^*$ is the same as the usual topology. Therefore $f^*$ is a lower semi-continuous convex function (\cite{zalinescu2002convex} Theorem 2.3.1). $f^*$ is proper because $f$ is proper.
\end{proof}
\begin{fact}\label{fact: double-f}
$f^{**} = f$.
\end{fact}
\begin{proof}
According to Fact~\ref{fact: r-hausdorff} and Fact~\ref{fact: f}, $f$ is a proper lower semi-continuous convex function on a Hausdorff locally convex space, therefore by Theorem~\ref{thm: fm}, we have $f^{**} = f$.
\end{proof}
\begin{fact}\label{fact: f-nondec}
$f^*$ is non-decreasing.
\end{fact}
\begin{proof}
This is because by assumption $f(x) = \infty$ for any $x < 0$, and then by Fact~\ref{fact: double-f}, $f^{**}(x) = \infty$ for any $x < 0$. This means $f^*$ is non-decreasing. 
\end{proof}

\begin{fact}\label{fact: f-sup}
$\sup_{t\in\bbR}\,t - f^*(t) = 0$.
\end{fact}
\begin{proof}
This is because by assumption $f(1) = 0$, and then by Fact~\ref{fact: double-f}, $f^{**}(1) = 0$. Therefore by the definition of $f^{**}$, we have 
$
    \sup_{t\in\bbR}\,t\cdot 1 - f^*(t) = 0.
$
\end{proof}

We will need the following definition and result from \cite{rockafellar1968integrals}, which we note to be {\em  simplied} because in our case $Q$ is a probability measure (instead of a $\sigma$-finite measure in their case) and we only consider real-valued functions (instead of vector-valued function in their case).
\begin{definition}[decomposable , \cite{rockafellar1968integrals} simplified]\label{def: decomp}
 We say a set of real-valued measurable functions over  $(\Omega, \Sigma)$ is decomposable if 
\begin{itemize}
    \item[(a)] $L\supseteq B(\Omega, \Sigma)$;
    \item[(b)] for any $u\in L$ and $E\in\Sigma$,  $u\cdot 1_E \in L$.
\end{itemize}

\begin{theorem}[\cite{rockafellar1968integrals}, corollary of Theorem 2, simplified]\label{thm: roc68}
Let $Q\in \calP_c(\Omega, \Sigma)$. Suppose $L$ and $L'$ are decomposable and for any $u\in L $ and $u'\in L'$ the function $u\cdot u'$ is integrable w.r.t. $Q$, $g: \bbR\to (-\infty, \infty]$ is a lower semi-continuous proper convex function, then for any $u'\in L'$
\begin{align*}
     \int_{\Omega} g^*(u')\diff Q = \sup_{u\in L} \int_{\Omega}u\cdot u'\diff Q - \int_{\Omega}g(u)\diff Q
\end{align*}
\end{theorem}
\end{definition}

\section{Proof of Theorem \ref{thm: extension}}
Note that by the Radon-Nikodym theoremm (\cite{folland1999real} Theorem 3.8), for each bounded and  countably additive signed measure on $(\Omega, \Sigma)$ that is absolutely continuous w.r.t. $Q$, there is an element in $L^1(Q)$, denoted by $\frac{\diff P}{\diff Q}$ (the Radon-Nikodym derivative), such that 
\begin{align}
    \diff P = \frac{\diff P}{\diff Q}\diff Q.\label{eq: radon}
\end{align}
Therefore, 
\begin{align*}
\sup_{h\in B(\Omega, \Sigma)} \bbE_{P}[h] - \bbE_{ Q}[f^*(h)]
    &\stackrel{(a)}{=} \sup_{h\in B(\Omega, \Sigma)}\int_{\Omega} h\cdot \frac{\diff P}{\diff Q}\diff Q - \int_{\Omega}f^*(h)\diff Q\\
    &\stackrel{(b)}{=} \int_{\Omega} f^{**}\left(\frac{\diff P}{\diff Q}\right)\diff Q.
\end{align*}
Here (a) is from \eqref{eq: radon}. To see why (b) holds, note that $B(\Omega, \Sigma)$ and  $L^1(Q))$ are decomposable spaces (as defined in Definition~\ref{def: decomp}) such that for any $u\in B(\Omega, \Sigma)$ and $u'\in L^1(Q)$, the function $u\cdot u'$ is integrable w.r.t. $Q$, $f^*$ is lower semi-continuous proper convex function by Fact~\ref{fact: fstar}, and $f^{**} = f$ by Fact~\ref{fact: double-f}; therefore we can apply Theorem~\ref{thm: roc68} and get the equality.
\section{Proof of Theorem \ref{thm: main}}
Define the functional $\calR : B(\Omega, \Sigma)\to[-\infty, \infty]$ to be 
\begin{align}
             \calR (h) \bydef \inf_{b\in\bbR}\, \bbE_Q\left[f^*\left(h + b\right)\right] - b.\label{eq: R-definition}
\end{align}
We first show some properties related to $\calR$.

\begin{lemma}\label{lem: restrict-int}
If $E\in \Sigma$ and $Q(E) = 0$, then for any $h\in B(\Omega, \Sigma)$, $\calR(h\cdot 1_{\Omega \setminus E}) = \calR(h)$.
\end{lemma}
\begin{proof} Observe that
\begin{align*}
    \calR(h\cdot 1_{\Omega \setminus E}) &= \inf_{b\in\bbR}\, \bbE_{Q}\left[f^*\left(h\cdot 1_{\Omega\setminus E} + b\right)\right] - b\\
    & \stackrel{(a)}{=} \inf_{b\in\bbR}\, \int_{\Omega\setminus E} f^*\left(h\cdot 1_{\Omega\setminus E} + b\right)\diff Q - b\\
    & = \inf_{b\in\bbR}\, \int_{\Omega\setminus E} f^*\left(h + b\right)\diff Q - b\\
    & \stackrel{(b)}{=} \inf_{b\in\bbR}\, \bbE_{Q}\left[f^*\left(h + b\right)\right] - b\\
    &=\calR(h),
\end{align*}
where (a) and (b) are because $Q(E) = 0$.
\end{proof}

\begin{lemma}\label{lem: lower}
The function $h\mapsto \bbE_Q[f^*(h)]$ defined on $B(\Omega, \Sigma)$ is lower semi-continuous.
\end{lemma}
\begin{proof}
Note that $B(\Omega, \Sigma)$ is a decomposable space (as defined in Definition~\ref{def: decomp}), and for any $u, u'\in B(\Omega, \Sigma)$, the function $u\cdot u'$ is integrable w.r.t. $Q$, and $f$ is lower semi-continuous proper convex function by Fact~\ref{fact: f}, then according to Theorem~\ref{thm: roc68}, 
\begin{align}
    \bbE_Q[f^*(h)] = \sup_{u\in B(\Omega, \Sigma)} \int_{\Omega}u\cdot h\diff Q - \int_{\Omega}f(u)\diff Q.\label{eq: sup-semi}
\end{align}
Note that for each $u\in B(\Omega, \Sigma)$, the function $h\mapsto \int_{\Omega}u\cdot h\diff Q$ defined on $B(\Omega, \Sigma)$ is a continuous linear function, therefore the r.h.s. of \eqref{eq: sup-semi} is the supremum of linear continuous functions, hence a lower semi-continuous function.
\end{proof}

\begin{lemma}\label{lem: utility}
    For any $h\in B(\Omega, \Sigma)$, $c\in\bbR$, sequence $\left\{c_n\right\}$ in $\bbR$, sequence $\left\{h_n\right\}$ in $B(\Omega, \Sigma)$, sequence $\left\{b_n\right\}$ in $\bbR$, if $h_n\to h$, $c_n\to c$, and $\bbE_Q[f^*(h_n + b_n)] - b_n \leq c_n$ for every $n$, then $\left\{b_n\right\}$ has a convergent subsequence whose limit point $b^*$ satisfies
    \begin{align*}
        \bbE_Q[f^*(h + b^*)] - b^* \leq c.
    \end{align*}
\end{lemma}
\begin{proof}
We first prove that the sequence $\left\{b_n\right\}$ is bounded. We prove this by contradiction. Suppose $\left\{b_n\right\}$ is not bounded, since $h_n\to h$ and $h\in B(\Omega, \Sigma)$, we have that for any $t > 0$, there exists $n_t$ such that 
\begin{align}
\left|h_{n_t} + b_{n_t}\right| \geq t.
\end{align}
However, by assumpition $f$ is finite in a neighbourhood of $1$, along with Fact~\ref{fact: double-f}, this implies that $f^*(x) - x\to \infty$ as $|x|\to\infty$. Therefore we have 
\begin{align*}
    \bbE_Q\left[f^*(h_{n_t} + b_{n_t}) - (h_{n_t} + b_{n_t})\right]\to\infty\text{\ as $t\to\infty$.}
\end{align*}
Because $h_n\to h$ and $h\in B(\Omega, \Sigma)$, we have that $h_{n_t}$ is bounded for $t > 0$, therefore 
\begin{align*}
    \bbE_Q\left[f^*(h_{n_t} + b_{n_t}) - b_{n_t}\right]\to\infty\text{\ as $t\to\infty$,}
\end{align*}
which contradicts the assumption that $\bbE_Q[f^*(h_n + b_n)] - b_n \leq c_n$ for every $n$ because $c_n\to c$. 

Now by Bolzano-Weierstrass theorem, the bounded sequence $\left\{b_n\right\}$ has a convergent subsequence $\left\{b_{i_n}\right\}$, whose limit point we denote by $b^*$. Let $\epsilon > 0$ be any positive real number, we will show that
\begin{align*}
    \bbE_Q[f^*(h + b^*)] \leq c + b^* + \epsilon.
\end{align*}
Because by assumption $b_{i_n}\to b^*$ and $c_{i_n}\to c$ and $\bbE_Q[f^*(h_{i_n} + b_{i_n})] - b_{i_n} \leq c_{i_n}$ for every $n$, we have that for $n$ large enough
\begin{align*}
    \bbE_Q\left[f^*\left(h_{i_n} + b_{i_n}\right)\right] \leq c + b^* + \epsilon.
\end{align*}
Lemma~\ref{lem: lower} says that the function $h\mapsto\bbE_Q\left[f^*\left(h\right)\right]$ is lower semi-continuous, since $h_{i_n} + b_{i_n} \to h + b^*$, this implies that 
\begin{align*}
    \bbE_Q\left[f^*\left(h + b^*\right)\right] \leq c + b^* + \epsilon.
\end{align*}
Because we can choose $\epsilon$ to be arbitrarily small, we can conclude that 
\begin{align*}
    \bbE_Q[f^*(h + b^*)] \leq c + b^*.
\end{align*}
\end{proof}

\begin{lemma}\label{lem: attained}
The infimum in the definition of $\calR$ as in \eqref{eq: R-definition} can be attained.
\end{lemma}
\begin{proof}
Let $t = \inf_{b\in\bbR}\bbE_Q\left[f^*\left(h + b\right)\right] - b$. We need to show that there exists $b^*$ such that $\epsilon_n\to 0$ and
\begin{align*}
    \bbE_Q\left[f^*\left(h + b^*\right)\right] - b^* = t.
\end{align*}
By the definition of infimum, there exists a sequence $\left\{\epsilon_n\right\}$ in $\bbR$ and a sequence $\left\{b_n\right\}$ in $\bbR$ such that 
\begin{align*}
    \bbE_Q\left[f^*\left(h + b_n\right)\right] - b_n \leq t + \epsilon_n
\end{align*}
Applying Lemma~\ref{lem: utility}, where $c_n = t + \epsilon_n$, $c = t$, and $h_n = h$, we have that there exists a subsequence of $\left\{b_n\right\}$ whose limit point $b^*$ satisfies
\begin{align*}
    \bbE_Q\left[f^*\left(h + b^*\right)\right] - b^* \leq t.
\end{align*}
Therefore the infimum in the definition of $\calR$ is attained at $b^*$.

\end{proof}

\begin{lemma}\label{lem: 4properties}
The functional $\calR$ has the following properties
\begin{enumerate}
    \item[$(\calR 1)$]$\calR$ is lower semi-continuous,
    \item[$(\calR 2)$]$\calR$ is convex,
    \item[$(\calR 3)$]For any $h, h'\in B(\Omega, \Sigma)$, if $h\leq h'$, then $\calR(h) \leq \calR(h')$. 
    \item[$(\calR 4)$]If $C\in B(\Omega, \Sigma)$ is a constant function, then $\calR(C) = C$.
\end{enumerate}
\end{lemma}
\begin{proof}
We will prove ($\calR1$)-($\calR 4$) separately:
\item\paragraph{Proof of ($\calR 1$). }
We need to show that for any $t\in\bbR$, $h\in B(\Omega, \Sigma$), and any sequence $\left\{h_n\right\}$ in $B(\Omega, \Sigma)$ such that $h_n\to h$ and $\inf_{b\in\bbR}\bbE_Q\left[f^*\left(h_n + b\right)\right] - b \leq t$ for any $n$, we have that 
\begin{align}
\inf_{b\in\bbR}\, \bbE_Q\left[f^*\left(h + b\right)\right] - b \leq t.\label{eq: implied}
\end{align}
Lemma~\ref{lem: attained} guarantees that there exists a sequence $\left\{\epsilon_n\right\}$ in $\bbR$ and a seqeunce $\left\{b_n\right\}$ in $\bbR$ such that for any $n$,
\begin{align*}
    \bbE_Q\left[f^*\left(h_n + b_n\right)\right] - b_n \leq t + \epsilon_n.
\end{align*}
Applying Lemma~\ref{lem: utility}, where $c_n = t + \epsilon_n$, $c = t$, we have that there exists $b^*\in\bbR$ such that 
\begin{align*}
    \bbE_Q\left[f^*\left(h + b^*\right)\right] - b^* \leq t,
\end{align*}
which will imply \eqref{eq: implied}.

\item\paragraph{Proof of ($\calR 2$). } For any $\lambda\in(0, 1)$, $\lambda' = 1 - \lambda$, $h, h'\in B(\Omega, \Sigma)$, we need to show that 
\begin{align}
    \inf_{b\in\bbR}\, \bbE_Q\left[f^*\left(\lambda h + \lambda' h' + b\right)\right] - b \leq \lambda\underbrace{\inf_{b\in\bbR}\left(\bbE_Q\left[f^*\left(h + b\right)\right] - b\right)}_{(a)}  + \lambda'\underbrace{\inf_{b\in\bbR}(\bbE_Q\left[f^*\left(h' + b\right)\right] - b)}_{(b)}.\label{eq: needToVerify}
\end{align}
If either (a) or (b) is infinite, then \eqref{eq: needToVerify} is trivially true; therefore we assume both of them to be finite. In this case, for any $\epsilon > 0$, there exists $b_1$ and $b_2$ such that 
\begin{align}
    \bbE_Q\left[f^*\left(h + b_1\right)\right] - b_1 \leq \inf_{b\in\bbR}(\bbE_Q\left[f^*\left(h + b\right)\right] - b) + \epsilon \label{eq: bound1}
\end{align}
and
\begin{align}
    \bbE_Q\left[f^*\left(h' + b_2\right)\right] - b_2 \leq \inf_{b\in\bbR}(\bbE_Q\left[f^*\left(h' + b\right)\right] - b) + \epsilon. \label{eq: bound2}
\end{align}
    We can see that
\begin{align*}
    \inf_{b\in\bbR}\, \bbE_{Q}\left[f^*\left(\lambda h + \lambda' h' + b\right)\right] - b &\leq \bbE_{Q}\left[f^*\left(\lambda h + \lambda' h' + \lambda b_1 + \lambda' b_2\right)\right] - \lambda b_1 - \lambda' b_2\\
                                                                                          &\stackrel{(a)}{\leq} \bbE_{Q}\left[\lambda f^*\left( h + \lambda b_1\right) + \lambda' f^*\left( h' + \lambda b_2\right)\right] - \lambda b_1 - \lambda' b_2\\
                                                                                          &= \lambda\left(\bbE_{Q}\left[f^*\left( h + \lambda b_1\right) - b_1\right) + \lambda'\left(f^*\left( h' + \lambda b_2\right)\right] - b_2\right)\\
                                                                                          &\leq \lambda \inf_{b\in\bbR}\left(\bbE_{Q}\left[f^*\left(h + b\right)\right] - b\right) + \lambda'\inf_{b\in\bbR}\left(\bbE_{Q}\left[f^*\left(h' + b\right)\right] - b\right) + \epsilon
\end{align*}
where (a) is due to the convexity of $f^*$, and (b) is due to \eqref{eq: bound1} and \eqref{eq: bound2}. Since $\epsilon$ can be made arbitrarily small, we have that \eqref{eq: needToVerify} is true.

\item\paragraph{Proof of ($\calR 3$). } For any $h, h'\in B(\Omega, \Sigma)$ such that $h \leq h'$, we have
\begin{align*}
    \calR(h)= \inf_{b\in\bbR}\, \bbE_{Q}\left[f^*\left(h + b\right)\right] - b\stackrel{(a)}{\leq} \inf_{b\in\bbR}\, \bbE_{Q}\left[f^*\left(h' + b\right)\right] - b = \calR(h'),
\end{align*}
where (a) is from Fact~\ref{fact: f-nondec}.

\item\paragraph{Proof of ($\calR 4$). } Note that 
\begin{align}
    \calR(C) &= \inf_{b\in\bbR}\, \bbE_{Q}\left[f^*\left(C + b\right)\right] - b\nonumber\\
             &= \inf_{b\in\bbR}\, f^*\left(C + b\right) - (C + b) + C\nonumber\\
             &= \inf_{t\in\bbR}\left(f^*\left(t\right) - t\right) + C\label{eq: useProperty}\\
             &\stackrel{(a)}{=} C\nonumber,
\end{align}
where (a) is from Fact~\ref{fact: f-sup}.
\end{proof}

Returning to our main proof, from ($\calR1$) and ($\calR2$), $\calR$ is a lower semi-continuous convex function over $B(\Omega, \Sigma)$. ($\calR4$) implies that $\calR(0) = 0$, therefore $\calR$ is finite at some point. Now we know that $\calR$ is a lower semi-continuous convex function that is finite at some point, therefore $\calR$ is a proper lower semi-continuous convex function. According to Theorem~\ref{thm: dual}, the conjugate of $\calR$ can be defined on $ba(\Omega, \Sigma)$, written as
\begin{align*}
    \calR^*(\mu) = \sup_{h\in B(\Omega, \Sigma)} \int_{\Omega} h \diff \mu - \calR(h). 
\end{align*}
Therefore from Fact~\ref{fact: b-hausdorff} and Theorem~\ref{thm: fm} we can conclude that for any $h\in B(\Omega, \Sigma)$,
\begin{align}
\calR (h) = \calR^{**}(h) = \sup_{P'\in ba(\Omega, \Sigma)}\bbE_{P'}\left[X\right] - \calR ^*\left(P'\right).\label{eq: arrive-at}
\end{align}

We will need the following lemma regarding $\calR^*$.

\begin{lemma}\label{lem: Rdual}
For any $P'\in ba(\Omega, \Sigma)$,
\begin{align*}
    \calR^*\left(P'\right) = \begin{dcases}
    \bar{D}_f(P'||Q), &\text{if $P'\in \calP(\Omega, \Sigma)$ and $P'\ll Q$},\\
    \infty, &\text{otherwise}.
    \end{dcases}
\end{align*}
\end{lemma}

\begin{proof}
For any $P'\in ba(\Omega, \Sigma)$, there are four possibilities:
\item\paragraph{Case 1.} If $P'$ is such that there exists $E\in\Sigma$ with $P'(E) < 0$, then for any $k > 0$, 
\begin{align*}
    \calR^*\left(P'\right)  &= \sup_{h\in B(\Omega, \Sigma)} \bbE_{P'}\left[h\right] - \calR(h)\\
    &\geq \bbE_{P'}\left[-k \cdot 1_E\right] - \calR(-k \cdot 1_E)\\
    &\stackrel{(a)}{\geq} -k\cdot P'(E) - \calR(0)\\
    &\stackrel{(b)}{=} -k\cdot P'(E) \\
    &\to \infty \text{\ as $k\to\infty$},
\end{align*}
Here (a) is due to ($\calR 3$) since $-k\cdot 1_E \leq 0$, and (b) is due to ($\calR 4$) with $C = 0$.

\item\paragraph{Case 2.}
If $P'$ is nonnegative and $P'\not\ll Q$, then there exists $E\in\Sigma$ such that $P'(E) > 0$ and $Q(E) = 0$, and for any $k > 0$, 
\begin{align*}
    \calR^*\left(P'\right)  &= \sup_{h\in B(\Omega, \Sigma)} \bbE_{P'}\left[h\right] - \calR(h)\\
    &\geq \bbE_{P'}\left[k \cdot 1_E\right] - \calR(k \cdot 1_E)\\
    &\stackrel{(a)}{=} k\cdot P'(E) - \calR(0)\\
    &\stackrel{(b)}{=} k\cdot P'(E) \\
    &\to \infty \text{\ as $k\to\infty$},
\end{align*}
where (a) is from Lemma~\ref{lem: restrict-int} and (b) is due to ($\calR 4$) with $C = 0$.

\item\paragraph{Case 3. }If $P'$ is non-negative, $P'\ll Q$, and $P'(\Omega)\neq 1$, then for any $k\in\bbR$, 
\begin{align*}
    \calR^*\left(P'\right)  &= \sup_{h\in B(\Omega, \Sigma)} \bbE_{P'}\left[h\right] - \calR(h)\\
    &\geq \bbE_{P'}\left[k\right] - \calR(k)\\
    &\stackrel{(a)}{=} k\bbE_{P'}[1] - k\\
    &\to \infty \text{\ either as $k\to\infty$ or as $k\to -\infty$},
\end{align*}
where (a) is due to ($\calR 4$). 

\item\paragraph{Case 4. }If $P'\in \calP(\Omega, \Sigma)$ and $P'\ll \calP(\Omega, \Sigma)$, then
\begin{align*}
    \calR^*\left(P'\right) &= \sup_{h\in B(\Omega, \Sigma)} \bbE_{P'}\left[h\right] - \calR(h)\\
                            &= \sup_{h\in B(\Omega, \Sigma)} \bbE_{x\sim P'}\left[h(x)\right] -\inf_{b\in\bbR}\left(\bbE_{Q}\left[f^*\left(h + b\right)\right] - b\right)\\
                                                 &= \sup_{h\in B(\Omega, \Sigma)} \bbE_{P'}\left[h\right] +\sup_{b\in\bbR}\left(-\bbE_{Q}\left[f^*\left(h + b\right)\right] + b\right)\\
                                                 &= \sup_{h\in B(\Omega, \Sigma)}\sup_{b\in\bbR} \bbE_{ P'}\left[h+b\right] -\bbE_{ Q}\left[f^*\left(h + b\right)\right]\\
                                                 &= \sup_{h\in B(\Omega, \Sigma)} \bbE_{ P'}\left[h\right] -\bbE_{ Q}\left[f^*\left(h\right)\right]\\
                                                 &\stackrel{(a)}{=} \bar{D}_f(P'||Q).
\end{align*}
\end{proof}

Therefore, for any $h\in B(\Omega, \Sigma)$, 
\begin{align}
             \calR (h) &\stackrel{(a)}{=} \sup_{P'\in ba(\Omega, \Sigma)}\bbE_{P'}\left[h\right] - \calR^*(P')
             \stackrel{(b)}{=} \sup_{P'\in\calP(\Omega, \Sigma)}\bbE_{P'}\left[h\right] - \bar{D}_f(P'||Q),\label{eq: riskdual2}
\end{align}
where (a) is due to \eqref{eq: arrive-at} and (b) is due to Lemma~\ref{lem: Rdual}.

Therefore, 
\begin{align}
  D_{f, \calH}(P||Q) &=  \sup_{h\in B(\Omega, \Sigma)} \bbE_{x\sim P}[h(x)] - \bbE_{x\sim Q}[f^*(h(x))] - \calH(h)\\
  &= \sup_{h\in B(\Omega, \Sigma), b\in\bbR}\left( \bbE_{x \sim P}[h(x) + b] - \bbE_{x\sim Q}\left[f^*\left(h(x) + b\right)\right] - \calH(h + b)\right)\nonumber\\
    &\stackrel{(a)}{=} \sup_{h\in B(\Omega, \Sigma), b\in\bbR}\left( \bbE_{x \sim P}[h(x)] + b - \bbE_{x\sim Q}\left[f^*\left(h(x) + b\right)\right] - \calH(h)\right)\nonumber\\
                        &\stackrel{(b)}{=} \sup_{h\in B(\Omega, \Sigma)}\left( \bbE_{P}[h] - \calR(h) - \calH(h)\right)\nonumber\\      
                        &\stackrel{(c)}{=} \sup_{h\in B(\Omega, \Sigma)}\left(\bbE_{P}[h] - \sup_{P'\in\calP(\Omega, \Sigma)}\left(\bbE_{P'}\left[h\right] - \bar{D}_f(P'||Q)\right) - \calH(h)\right)\nonumber\\
                       &= \sup_{h\in B(\Omega, \Sigma)}\inf_{P'\in\calP(\Omega, \Sigma)}  \bbE_{P}[h] -  \bbE_{P'}[h]  - \calH(h) + \bar{D}_f(P'||Q)\nonumber\\
                         &\stackrel{(d)}{=}\inf_{P'\in\calP(\Omega, \Sigma)} \sup_{h\in B(\Omega, \Sigma)} \bbE_{P}[h] -  \bbE_{P'}[h]  - \calH(h) + \bar{D}_f(P'||Q)\nonumber\\
                         &\stackrel{(e)}{=}\inf_{P'\in\calP(\Omega, \Sigma)} \calH^*(P - P') + \bar{D}_f(P'||Q)\nonumber\\ &\stackrel{(f)}{=}\inf_{\substack{P'\in\calP(\Omega, \Sigma)\\P'\ll Q}} \calH^*(P - P') + \bar{D}_f(P'||Q),\nonumber
\end{align}
where (a) is because by assumption $\calH$ is shift invariant and the constant $b$ can be moved out of the expectation; (b) is due to the definition of $\calR$ in \eqref{eq: R-definition}; (c) is due to \eqref{eq: riskdual2}; (d) is due to Lemma~\ref{lem: kyfan} below; (e) is by the definition of convex conjugate; and (f) is because we have shown that $D_f(P'||Q) = \calR^*(P')$ and $\calR^*(P')$ is infinite  when $P'\not\ll Q$.

\begin{lemma}\label{lem: kyfan}
Define $F: \calP(\Omega, \Sigma)\times B(\Omega, \Sigma) \to(-\infty, \infty]$ to be
\begin{align*}
    F\left(P', h\right) &\bydef \bbE_{P}[h] -  \bbE_{P'}[h] - \calH(h) + \bar{D}_f(P'||Q).
\end{align*}
Then, 
\begin{align}
    \sup_{h\in H} \inf_{P'\in\calP(\Omega, \Sigma)} F\left(P', h\right) = \inf_{P'\in\calP(\Omega, \Sigma)} \sup_{h\in H} F\left(P', h\right). \label{eq: minimax-eq}
\end{align}
\end{lemma}
\begin{proof}
Here we will use Ky Fan's minimax theorem (\cite{fan1953minimax}, theorem 2), which says that \eqref{eq: minimax-eq} is true as long as we can equip $\calP(\Omega, \Sigma)$ with certain topology such that \begin{enumerate}
\item[(a)]
$B(\Omega, \Sigma)$ and $\calP(\Omega, \Sigma)$ are convex.
\item[(b)]
$\calP(\Omega, \Sigma)$ is compact.
\item[(c)]
$F$ is convex and lower semi-continuous on $P'$. 
\item[(d)]
$F$ is concave $h$.
\end{enumerate}

We will first show what topology on $\calP(\Omega, \Sigma)$ we will use. Since $\calP(\Omega, \Sigma)\subseteq ba(\Omega, \Sigma)$, according to Theorem~\ref{thm: dual}, we can equip $\calP(\Omega, \Sigma)$ with the weak* topology induced by $B(\Omega, \Sigma)$. Henceforth when talking about (semi-)continuity and compactness, they are all with respect to this weak* topology.

We can see that (a) is obviously true and (d) is true because by assumption $\calH$ is convex. It remains to check conditions (b) and (c).

\item\paragraph{Condition (b). }
According to the Banach-Alaoglu theorem, the set 
\begin{align*}
\bar{\calP}\bydef \left\{P: P\in ba(\Omega, \Sigma), P(\Omega)\leq 1\right\}
\end{align*}
is compact and $\calP(\Omega, \Sigma)\subseteq \bar{\calP}$. It suffices to show that $\calP(\Omega, \Sigma)$ is a closed subset of $\bar{\calP}$.
In fact, we have that
\begin{align*}
    \calP(\Omega, \Sigma) = \left(\bigcap_{U\in\Sigma} \left\{P\in\bar{\calP}: \bbE_{P'}[1_U]\geq 0\right\}\right) \cap \left\{P\in\bar{\calP}: \bbE_{P'}[1] = 1\right\}.
\end{align*}
Because both the function $1$ and the functions $1_U$ are in $B(\Omega, \Sigma)$, we can see that $\calP(\Omega, \Sigma)$ is the intersection of the preimages of closed sets under continuous functions, therefore is also closed. 
\item\paragraph{Condition (c). }
For any $h\in H$, the function $P' \mapsto  \bbE_{P'}[h]$ is a linear and continuous function; and the function $P' \mapsto \bar{D}_f(P'||Q)$ is convex and lower semi-continuous because according to its definition it is a supremum over a set of linear and continuous functions of $P'$. Therefore, $F$ is convex and lower semi-continuous on $P'$.
\end{proof}

\end{document}